\documentclass[11pt]{article} %
\usepackage{fullpage}
\usepackage{amsthm}
\usepackage{amsmath, amssymb, natbib, graphicx, url}
\usepackage{color}
\usepackage{todonotes}
\usepackage{dsfont}
\usepackage{float}
\usepackage{caption}
\usepackage{subcaption}
\usepackage{xcolor}

\usepackage{hyperref}
\definecolor{refcol}{rgb}{0.1,0,0.6}
\hypersetup{colorlinks}
\hypersetup{linkcolor=refcol, citecolor=refcol}

\usepackage[nameinlink]{cleveref}

\usepackage{algorithm}%
\usepackage{algorithmicx}%
\usepackage{algpseudocode}%

\newcommand{\cA}{\mathcal{A}}

\newcommand{\E}{\mathbb{E}}

\newcommand{\I}{\mathbb{I}}

\usepackage{enumitem}
\usepackage[normalem]{ulem}

\newcommand{\pa}{P_a}
\newcommand{\target}{\Pi_N = \{\pi_1, \ldots, \pi_N\}}
\newcommand{\regret}{\mathcal R(\widehat \pi_n) := v(\pi_\ast)-v(\hat\pi_n)}
\newcommand\sigmaKL{\sigma_{\mathrm{KL}}}
\newcommand\KL{\mathrm{D}_{\mathrm{KL}}}
\DeclareMathOperator*{\argmax}{arg\,max}

\usepackage{mathtools}

\usepackage{algorithm,algpseudocode} 
\usepackage{authblk}

\title{Clustered KL-barycenter design for policy evaluation}

\author[1,$\ast$]{Simon Weissmann}
\author[2,$\ast$]{Till Freihaut}
\author[3]{Claire Vernade}
\author[2]{Giorgia Ramponi}
\author[1]{Leif Döring}
\date{\ }

\affil[1]{\normalsize
  Universit\"at Mannheim, Institute of Mathematics, 68138 Mannheim, Germany\\
  
\texttt{\{simon.weissmann\},\{leif.doering\}@uni-mannheim.de}
}
\affil[2]{\normalsize University of Zurich, Department of Informatics, 8006 Zurich, Switzerland\\

  \texttt{\{freihaut\},\{ramponi\}@ifi.uzh.ch}
}
\affil[3]{\normalsize
  University of Tübingen, 72074 Tuebingen, Germany\\
  
\texttt{claire.vernade@uni-tuebingen.de}
}

\affil[$\ast$]{\normalsize
These authors contributed equally to this work
}

\newtheorem{theorem}{Theorem}[section]

\newtheorem{corollary}[theorem]{Corollary}

\newtheorem{example}[theorem]{Example}

\newtheorem{lemma}[theorem]{Lemma}

\newtheorem{proposition}[theorem]{Proposition}
\newtheorem{remark}[theorem]{Remark}

\newtheorem{assumption}[theorem]{Assumption}

\begin{document}

\maketitle

\begin{abstract}
 In the context of stochastic bandit models, this article examines how to design sample-efficient behavior policies for the importance sampling evaluation of multiple target policies. From importance sampling theory, it is well established that sample efficiency is highly sensitive to the KL divergence between the target and importance sampling distributions. We first analyze a single behavior policy defined as the KL-barycenter of the target policies. Then, we refine this approach by clustering the target policies into groups with small KL divergences and assigning each cluster its own KL-barycenter as a behavior policy. This clustered KL-based policy evaluation (CKL-PE) algorithm provides a novel perspective on optimal policy selection. We 
prove upper bounds on the sample complexity of our method and demonstrate its effectiveness with numerical validation.
\end{abstract}

{\bf Keywords:} stochastic bandit, policy evaluation and selection, importance sampling, clustering.

\section{Introduction}
Importance Sampling (IS) \citep{mcbook} is a fundamental tool in Monte Carlo simulations, primarily used to estimate expectations under distributions other than the sampling one.
The IS literature traditionally addresses two key questions: (1) How to design an optimal importance sampling distribution for a given target distribution and (2) how to determine the required sample size for reliable estimation. 
In machine learning, IS is widely applied to evaluate new objectives using existing data. However, re-weighting data without ensuring proper alignment between the sampling and target distributions can result in high-variance estimators. As a result, research has focused on variance reduction techniques \citep[e.g.][]{bottou2013counterfactualreasoninglearningsystems, ConfidentOPE,sakhi2024logarithmic} and coverage assumptions, often at the cost of introducing estimation bias or inefficient sampling distributions. An important  application of IS in machine learning is off-policy evaluation (OPE) in (contextual) bandits \citep{wang2017optimaladaptiveoffpolicyevaluation,  Agarwal_2017, gabbianelli2023importanceweightedofflinelearningright}, where multiple target policies $\pi_1,\dots, \pi_N$ are evaluated using data collected under a behavior policy. Recently, there has been renewed interest in IS’s original principles, that is, directly addressing the problem of \emph{constructing a behavior policy} to reduce the variance of the estimator \citep{hanna2017behaviorsearch, papini2024policygradientactiveimportance, jain2024adaptiveexplorationdataefficientgeneral}. While prior work focus on designing behavior policies for single target distributions, less attention has been given to simultaneously estimating multiple expectations \citep{demangechryst2022efficientestimationmultipleexpectations}. This setting, where data collection should efficiently identify the most valuable policy, remains underexplored and lacks theoretical guarantees. The present article addresses the following questions:
\begin{center} \label{question} \emph{Q1: How should a behavior policy $\pi_b$ be designed to effectively evaluate a given set of target policies? Can it be efficient to use several behavior policies $\pi^{1}_b,...,\pi^{M}_b$?}
\end{center}

We study this question in the setting of stochastic bandits, defined by the tuple $(\cA, \pa)$, where $\cA$ is a finite action space and $\pa:\cA \to \mathbb{R}$ is a probability kernel mapping actions to rewards. We assume that all reward distributions are $R_*$-subgaussian, and denote the expected reward under action $a\in\mathcal A$ as
$Q(a) = \mathbb{E}[R(a)] := \int_{\mathbb{R}} x dP_a(x)$. 
A policy $\pi=(\pi(a))_{a\in\mathcal{A}}$ is a probability distribution over actions and its \emph{value} is defined as $v(\pi) = \sum_a \pi(a) Q(a)$. This article is driven by the best-policy selection problem, the identification from data of the best policy from a set $\Pi_N$ using a carefully chosen behavior policy. The IS estimator of a policy $\pi$ with respect to the sampling policy $\pi_b$ is defined as
\begin{equation}
    \label{eq:ISbandits}
    \widehat{v}_n(\pi) := \frac{1}{n} \sum_{t=1}^n \frac{\pi(A_t)}{\pi_b(A_t)} R_t,
\end{equation}
where the pairs $(A_t,R_t)$ are iid action-reward pairs obtained from the bandit model when playing the policy $\pi_b$. As mentioned previously, the choice of the behavior policy crucially impacts the IS estimator's performance. The method we suggest is based on insights from the IS community on the role of Kullback-Leibler (KL) divergences for sample-efficient IS. It turns out that using KL-barycenters is both a theoretically feasible and an easily implementable way to find behavior policies to evaluate a set of target policies. Since the KL-barycenter is used as behavior policy, we coin this method \emph{KL based policy evaluation (KL-PE)}.
We analyze its theoretical limitations and show that if the target policies lack sufficient structure, the constructed behavior policy may degrade in performance as the number of policies increases. To address this issue, we propose an extension, \emph{clustered KL based policy evaluation (CKL-PE)}, which clusters target policies based on probabilistic similarity and designs multiple behavior policies accordingly. This theoretically grounded approach enables efficient scaling to large policy sets while maintaining strong performance guarantees.

\paragraph{Outline.}
In Section~\ref{sec:approach}, we present a precise problem formulation and outline the high-level idea behind the design of the behavior policies. 
Section~\ref{sec:RelatedWork} reviews prior research on off-policy evaluation (OPE) and behavior policy optimization, highlighting key differences between our work and existing methods. Our main results on KL-based policy evaluation (KL-PE) and clustered KL-based policy evaluation (CKL-PE) are presented in Sections \ref{sec:Structured} and \ref{sec:ImprovedEvaluation} respectively. Our analysis focuses on the consequence of structural assumptions on the sample complexity of policy evaluation and selection.
Finally, we conclude with an empirical evaluation of our methods and discuss potential directions for future work.

\section{Formalization of the problem and approach}\label{sec:approach}

\paragraph{Best-policy selection.} In this section, we formalize the problem of interest, the so-called best-policy selection \citep[e.g.][]{sakhi2024logarithmic}. We are given a set of target policies $\Pi_N:=\{\pi_1,\dots,\pi_N\}\subset \Delta_{\mathcal A}$, where $\Delta_{\mathcal A}$ denotes the probability simplex over $\mathcal{A}$, and the goal is to identify the best policy,
\[ \pi_\ast = \argmax_{i=1,\dots,N}\ v(\pi_i)\,.\] 
To select the best policy, one must collect data by interacting with the bandit environment, and the aim is to use as few samples as possible to estimate the values $v(\pi_i)$. A straightforward naive approach is to collect $N$ independent datasets using each $\pi_i$ and compute directly the Monte Carlo estimators $\hat{v}^{\mathrm{MC}}(\pi_i)=\sum_{t=1}^{n_i}R_t^{(i)}/n_i$, $i=1,\dots,N$, where $(A_t^{(i)},R_t^{(i)})$ are sampled for each policy $\pi_i \in \Pi_N$ separately. These pairs arise from a two-stage stochastic experiment: first, an action is sampled by a policy, and then, conditional on the action, a reward is sampled from $P_{A_t}$. 
This is indeed inefficient as each data pair is only used once\footnote{We quantify this precisely in Section~\ref{sec:Structured}.}.  
Instead, we explore the use of IS estimators with respect to joint behavior policies $\pi_b^{1},\dots,\pi_b^M$, where $M\ll N$. Similarly, each of these behavior policies generate a dataset and we show how to use them to evaluate \emph{multiple} target policies.
In the end, the selected policy $\widehat{\pi}_n$ is a data-dependent random variable, obtained by maximizing the estimated values:
\begin{equation} \label{eq:bestpolicy_wo_cluster}
\widehat{\pi}_n := \argmax_{i=1,\dots,N}\  \widehat{v}_n^{(i)}(\pi_i),
\end{equation}  
where for each $i\in\{1,\dots,N\}$ the estimated value is constructed as IS estimators defined in \eqref{eq:ISbandits} using $n_j$ samples of one of the behavior policies $\pi_b^j$, $j=1,\dots,M$. The overall number of interactions with the bandit environment is given by $n=\sum_{j=1}^Mn_j$. While standard IS analysis focuses on the efficient estimation of a single expectation, our setting differs significantly. Our primary goal is to achieve a successful best policy identification rather than merely obtaining accurate value estimates. Poor choices of the behavior policy can lead to both underestimation and overestimation of different policies within our target set. Consequently, we aim to design a behavior policy  that enables accurate IS estimators.

The measure of performance in this setting is the \emph{regret} or \emph{excess risk}:
\begin{equation}
\label{eq:regret}
    \regret\,,
\end{equation}
which emphasizes that our focus is not only on the precise estimation of policy values but, more importantly, on correctly selecting the best policy within our target set $\Pi_N$.

\paragraph{Behavior policy design.}
For a single IS estimator, \cite{chatterjee2017samplesizerequiredimportance} established that the minimal sample size required for an IS estimator to achieve a low $L^1$-error is of the order $\exp(\mathrm{D}_{\mathrm{KL}}(\pi \Vert \pi_b))$, where 
\[ \mathrm{D}_{\mathrm{KL}}(\nu \Vert \mu) := \sum_{x \in \mathcal{X}} \nu(x) \log\left(\frac{\nu(x)}{\mu(x)}\right) \]
denotes the KL divergence between two probability distributions $\nu$ and $\mu$. Moreover, they demonstrate via a lower bound that this sample size is necessary to ensure that the IS estimator is close to its expectation with high probability. This insight contrasts with most analyses that determine the sample size based on the variance. Since variance-based approaches can lead to higher estimates of the required sample size than those based on the KL divergence, they may result in an unnecessarily large number of samples (see e.g. \cite[Example 1.4]{chatterjee2017samplesizerequiredimportance}). Further discussions on the role of the KL divergence in IS can be found in \citep{0b004835-9091-3512-9a35-689c1f1879e2,doi:10.1137/16M1093549,beh2023insightkullbackleiblerdivergenceadaptive}. 

Building on these findings, we design our behavior policies to maintain a low KL with respect to the set of target policies $\target$. To achieve this, we use a perhaps surprisingly simple yet important result. Provided the target policies assign positive probability to all actions (such as in the case of softmax policies), the KL-barycenter is defined as the solution of
\begin{equation}\label{eq:Opti_KLbarycenter}
\min_{\pi\in\Delta_{\mathcal A}}\ D_{\mathrm{KL}}(\Pi_N | \pi),\quad \text{where } D_{\mathrm{KL}}(\Pi_N | \pi):= \frac{1}{N}\sum_{i=1}^N D_{\mathrm{KL}}(\pi_n | \pi), 
\end{equation}

 The following result is an immediate consequence of~\citet{KLbarycenter} and guides a natural way to construct a well-behaved behavior policy.

\begin{lemma}
    \label{thm:KL-barycenter}
    Let $\target$ be a set of strictly positive sum-normalized  target policies. Then the arithmetic mean $\pi_{\mathrm{KL}}:= \frac{1}{N} \sum_{i=1}^N \pi_i$ solves \eqref{eq:Opti_KLbarycenter}. We call this policy the KL-barycenter.
\end{lemma}
In Lemma~\ref{lem:KLbarycenter} of the appendix we precisely formulate and prove the above statement. 

\paragraph{Policy Evaluation via KL-barycenter.} Lemma \ref{thm:KL-barycenter} suggests that the KL-barycenter 
\begin{equation}
\label{eq:barycenter}
    \pi_{\mathrm{KL}}:=\frac{1}{N}\sum_{i=1}^N \pi_i
\end{equation}
could be a good choice for a single behavior policy. It is important to emphasize that computing the KL-barycenter does not require any interaction with the bandit environment and is easy to compute. We summarize the KL-barycenter based approach in Algorithm~\ref{alg:SPE}. 
\begin{algorithm}[H]
\begin{algorithmic}[1]
\caption{Best policy selection using KL-barycenter design}
\label{alg:SPE}
\Require Target policies $\Pi_N$, sample size $n$
\Ensure Selected best policy $\widehat{\pi}_n$ \smallskip
\State Compute the KL-barycenter $\pi_{\mathrm{KL}}$ 
\State Collect dataset $D$ of size $n$ using $\pi_{\mathrm{KL}}$ 
\State Compute $\widehat{v}_n(\pi_i)$ according to \eqref{eq:ISbandits}: $\widehat{v}_n(\pi_i) = \frac{1}{n} \sum_{t=1}^n \frac{\pi_i(A_t)}{\pi_b(A_t)} R_t$  for all $\pi_i \in \Pi_N.$
\State Select best policy according to \eqref{eq:bestpolicy_wo_cluster}: $\widehat{\pi}_n := \argmax_{i=1,\dots,N}\  \widehat{v}_n^{(i)}(\pi_i).$
\end{algorithmic}
\end{algorithm}

\paragraph{Clustered KL-barycenter design.} 
It turns out that in certain scenarios using only one behavior policy is insufficient. In particular, in Section~\ref{sec:Structured}, we prove a lower bound that scales with the highest importance weight, that depends on the KL-barycenter. To mitigate this effect, we then propose a refined algorithm that scales up effectively. First, we separate the set of target policies $\Pi_N$ into multiple clusters based on similarities and then apply the KL-barycenter design within each cluster. We give an informal description of this approach in Algorithm~\ref{alg:ISPE} and provide full details in Sections~\ref{sec:clusteredevaluation} and \ref{sec:clustering}.

\begin{algorithm}[H]
\begin{algorithmic}[1]
\caption{Best policy selection using clustered KL-barycenter design}
\label{alg:ISPE}
\Require Target policies $\Pi_N$, number of clusters $M$, sample size $n$ 
\Ensure Selected best policy $\widehat{\pi}_n$\smallskip

\State Cluster target policies $\Pi_N$ into $M$ clusters (e.g. with Algorithm~\ref{alg:Hellinger_cluster}) 

\State Compute KL-barycenter behavior policies $\pi^{(j)}_{{\mathrm{KL}}}$ for each cluster $j=1, \dots,M$ 
\State Collect dataset $D$ of size $n$ using all $\pi^{(j)}_{b}$ 
\State For reach cluster $K_j$, compute $\widehat{v}_n^{(j)} (\pi_i)\,$ according to \eqref{eq:clusterIS}: $\widehat{v}_n^{(j)} (\pi_i) = \frac{1}{n_j}\sum_{t=1}^{n_j} \frac{\pi_i(A_t^{(j)})}{\pi_{\mathrm{KL}}^{(j)}(A_t^{(j)})} R_t^{(j)}\,$ for all $\pi_i \in K_j$;
\State Select best policy according to \eqref{eq:bestpolicy_wo_cluster}:
$
\hat{\pi}_n \gets \argmax_i \hat{v}_n^{(j)}(\pi_i).
$
\end{algorithmic}
\end{algorithm}

\section{Related Work}
\label{sec:RelatedWork}

\paragraph{Behavior policy optimization.} Our work is   closely related to the emerging field of behavioral policy optimization. This is an experimental design problem, where the objective is to find an optimal data collection strategy for evaluating policies. Such methods have been particularly relevant in offline reinforcement learning \citep{hanna2017behaviorsearch}. In this context, they propose a theoretical and practical method for selecting a behavior policy that effectively evaluates a single target policy. More recently, \citet{jain2024adaptiveexplorationdataefficientgeneral} extend this idea to the evaluation of multiple policies in an MDP setting by sequentially designing a behavior policy that minimizes the average Mean-Squared Error(MSE) between the behavior policy and a set of target policies. Additionally, in another line of research, \citet{papini2024policygradientactiveimportance} introduce the behavioral policy optimization problem in the context of policy gradient methods, aiming to collect data that minimizes the variance of policy gradient estimates.
However, these prior works focus primarily on minimizing the mean squared error. In contrast, our work takes a information-theoretic perspective by considering the KL divergence, evaluating the performance on the best policy selection problem and give a theoretical in-depth analysis.

\paragraph{Off-policy evaluation.} Our work is somewhat orthogonal to the well-studied problem of off-policy evaluation (OPE) \citep{Li_2011, bottou2013counterfactualreasoninglearningsystems, JMLR:v16:swaminathan15a}. In OPE, data has already been collected using a fixed and often known behavior policy. IS estimators are then used to evaluate the set of target policies. However, as the behavior policy is fixed and can be rather arbitrarily, there can be a mismatch between the target and behavior policies. The focus of these works is to refine the estimator instead of selecting a suitable behavior policy.  Several techniques have been proposed to mitigate variance due to policy mismatch. Clipping IS weights is a widely used approach \citep{Truncated, pmlr-v37-thomas15, bottou2013counterfactualreasoninglearningsystems}, as is introducing a pessimistic bias into the estimator \citep{JMLR:v16:swaminathan15a, london2020bayesiancounterfactualriskminimization, jin2022pessimismprovablyefficientoffline, rashidinejad2023bridgingofflinereinforcementlearning, jin2024policylearningwithoutoverlap}. Another variance-reduction method is Self-Normalized IS, which stabilizes estimates while maintaining practical effectiveness \citep{ConfidentOPE, Hesterberg01051995}. Many of these approaches assume uniformly bounded importance weights, an assumption recently relaxed by~\citet{gabbianelli2023importanceweightedofflinelearningright} through the introduction of an exploration parameter $\gamma$, which implicitly constrains importance weights. This perspective aligns with the idea of selecting a "safe" behavior policy, which is formalized in this work.

\section{Structured Policy Evaluation with KL-Barycenter}
\label{sec:Structured}
In this section we introduce structured policy evaluation with KL-barycenter. Motivated by the discussion in Section \ref{sec:approach} we use the KL-barycenter to evaluate a finite number of policies. The idea is easy to implement (the barycenter is a mixture of the target policies, Lemma~\ref{thm:KL-barycenter}) and the excess risk is straightforward to analyze. The section is crucial to understand the benefit of improved structured policy evaluation (CKL-PE) introduced in Section \ref{sec:ImprovedEvaluation}.
\subsection{Regret Upper Bound for KL-Barycenters}
Let us consider the KL-barycenter behavior policy $\pi_{\mathrm{KL}}:=\frac{1}{N}\sum_{i=1}^N \pi_i$ and the corresponding IS estimator for arbitrary policies $\pi$:
\[\widehat v_n(\pi) := \frac1n \sum_{t=1}^n w_{\mathrm{KL}}^{\pi}(A_t) R_t\,, \quad w_{\mathrm{KL}}^{\pi}(a) := \frac{\pi(a)}{\pi_{\mathrm{KL}}(a)} \] 
A crucial simple fact on importance sampling with respect to $\pi_{\textrm{KL}}$ is the boundedness of importance weights:
\begin{lemma}\label{lem:sigmaKL}
    If $w_{\mathrm{KL}}^{\pi}(\cdot) := \frac{\pi(\cdot)}{\pi_{\mathrm{KL}}(\cdot)}$ are the importance sampling weights of $\pi\in\Pi_N$ with respect to $\pi_{\textrm{KL}}$, then the maximal weight 
    \begin{equation} \sigma_{\textrm{KL}}:=\max_{\pi\in\Pi_N}\max_{a\in\mathcal A} w_{\mathrm{KL}}(a)
    \end{equation}
    is bounded by $N$.  
\end{lemma}
\begin{proof}
    This follows immediately from the definition of $\pi_{\textrm{KL}}$.
\end{proof}
The boundedness of importance sampling weights simplifies the analysis a lot and corresponds to a uniform coverage assumption, a common requirement in the OPE literature, e.g.  \cite{wang2023oracleefficientpessimismofflinepolicy}. When working with KL-barycenters the crucial quantity is  $\sigma_{\textrm{KL}}$, a constant that can be interpreted as a measure of how well the behavior policy aligns with the set of target policies. As such, controlling $\sigma_{\textrm{KL}}$ is of central importance. We will come back to this topic in Section \ref{sec:ImprovedEvaluation} where we suggest clustering methods to strongly decrease the importance sampling weights.

Since the simple form of the KL-barycenter implies bounded weights, the regret (as defined in Equation \ref{eq:regret}) of the IS estimator is easy to estimate. 
Assuming $R_*$-subgaussian rewards, Hoeffding's inequality immediately implies
\[ \mathbb P(\widehat v_n(\pi)-v(\pi) \ge \varepsilon) \le \exp\left(-\frac{2\varepsilon^2 n}{(R_*\sigma_{\mathrm{KL}})^2}\right)\,.\]
Hence, given $\delta\in(0,1)$, with probability at least $1-\delta$ we have
\[v(\pi)\ge \widehat v_n(\pi) - \frac{1}{\sqrt{2n}} R_\ast \sigma_{\mathrm{KL}} \sqrt{\log\left(\frac{N}{\delta}\right)}\]
for all $\pi\in\Pi_N$. Similarly, it holds that
\[\widehat v_n(\pi)\ge v(\pi) - \frac{1}{\sqrt{2n}} R_\ast \sigma_{\mathrm{KL}} \sqrt{\log\left(\frac{N}{\delta}\right)}\]
for all $\pi\in\Pi_N$. These computations outline the key steps in proving the following proposition. The detailed proof is given in Appendix~\ref{proof:regret_unclustered}.

\begin{proposition}\label{prop:regret_unclustered}
    For arbitrary $\delta\in(0,1)$ and $\varepsilon >0$, let \[ n \ge n(\varepsilon,\delta) = \frac{2R_\ast^2 \sigma_{\mathrm{KL}}^2\log\left(\frac{N}{\delta}\right)}{\varepsilon^2}\]
    many iid.~pairs $(A_t,R_t)_{t=1}^n$ of action and rewards be generated by $\pi_{\mathrm{KL}}$ and $\widehat \pi_n$ defined in \eqref{eq:bestpolicy_wo_cluster}. Then 
    \[\mathbb P(\mathcal R(\widehat \pi_n) < \varepsilon) \ge 1-\delta\,.\]
\end{proposition}

From the above result, we observe that the required sample size $n(\varepsilon,\delta)$ for achieving a small regret with high probability scales with $\sigma_{\mathrm{KL}}$, reaching its worst case when $\sigma_{\mathrm{KL}} = N$. This means that when the behavior policy is poorly aligned with the target policies, the sample complexity increases linearly with the number of policies, making evaluation inefficient. To understand this inefficiency, consider the naive approach of evaluating each target policy individually. In this case, $\sigmaKL=1$, but we would need to sample separately for each of the policies, leading to a total sample complexity of $n\ge\frac{2NR_\ast^2\log(\frac{N}{\delta})}{\varepsilon^2}$. Therefore, KL-PE is more sample efficient only if $\sigmaKL\leq \sqrt{N}$. The following lower bound further underscores that this condition does not always hold, emphasizing the need for additional constraints on the behavior policy.

\subsection{Lower bound}
We first give an example in which the KL-barycenter is a poor choice as a behavior policy.
\begin{example}
Let us consider a three-armed Bernoulli bandit problem, with the following reward distributions $p_1 = 0.9$ and $0.9> p_2 \geq p_3$. We first construct a target policy set of size 3
    \[\pi_1 = (0.99, 0.005 , 0.005),\quad \pi_2 = (0.005, 0.49, 0.505),\quad \pi_3 = (0.005, 0.505,0.505).\]
    It is easy to verify that the KL-barycenter is a uniform policy, i.e. $\pi_{\mathrm{KL}} = (1/3,1/3,1/3)$. If we now increase the number of target policies that all have probability $0.005$ on arm $1$, then we have for a target policy set of size $N$, with $\pi_{\mathrm{KL}}(a_1) = \frac{1}{N}(0.99 + (N-1) \cdot 0.005)$. In particular, $\pi_{\mathrm{KL}}(a_1)$ approaches $0.005$, meaning that $\sigmaKL$ approaches $198$ and the required sample size scales with $198^2 = 39204.$
\end{example}
These findings illustrate that the KL-barycenter as a behavior policy has limitations, when the set of target policies has a unpleasant structure and the amount of target policies increases. The following lower bound generalizes these findings. The proof can be found in Appendix~\ref{proof:lowerbound}.

\begin{proposition}
\label{prop:lowerbound}
    For arbitrary $N\ge3$ there exists a multi-armed bandit model $(\mathcal A^{(N)},P_{R^{(N)}})$ and a set of target policies $\Pi^{(N)}_N = \{\pi_1^{(N)},\dots, \pi_N^{(N)}\}$ such that $\sigma_{\mathrm{KL}}^{(N)}\ge \sigma_N = \frac{N}{2}$ and
    \[ \mathbb P(\mathcal R(\widehat\pi_n)>\varepsilon) \ge \frac{1}{\sqrt{2n}} \exp\left(-\frac{n}{2\sigma^2_N}\right)\]
    for all $\varepsilon>0$ sufficiently small.
\end{proposition}

\begin{remark}
    In Proposition~\ref{prop:lowerbound}, we stated that \( \varepsilon \) needs to be sufficiently small. Our proof is based on a specific construction of a two-armed stochastic bandit with deterministic rewards $r_1>r_2$, and a set of target policies. By "sufficiently small" we mean that $\varepsilon\le \Delta$ where $\Delta$ represents the gap between the best policy and the remaining policies, which all share the same value. In our construction, this gap can be explicitly expressed as
    \[
    \Delta = \left(\frac{1}{N} - \frac{1}{N(N-1)}\right) \left(-\frac{2}{N} +1\right) r_1 > 0, \quad \forall N\geq3.
    \]
    Consequently, the gap scales with $r_1>0$ but remains independent of its specific value.
\end{remark}
The lower bound demonstrates that indeed there exists a set of target policies for a bandit model, such that $\sigmaKL$ is close the its worst case $N$ and that the probability of incurring a large regret scales with the highest importance weight, that depends on the KL-barycenter. The observed findings indicate that to ensure a reasonable size $\sigma_{\mathrm{KL}},$ additional constraints must be imposed on the behavior policy $\pi_{\mathrm{KL}}$. Without such constraints, the sample complexity can grow unfavorably, highlighting the necessity of a structured behavior policy design. Additionally, in Appendix~\ref{app:scaling_up}, we explore properties of the KL-barycenter as the number of target policies approaches infinity. 

\subsection{On the Relationship Between the KL-Barycenter and the Target Policies}

Despite the KL-barycenter's promise of bounded importance weights, our previous analysis revealed that the bound may be quite large. Potentially, this may lead to poor performance in policy evaluation unless certain structural assumptions are made about the set of target policies. In the following section, we show that the policy evaluation can be significantly improved by assuming proximity between each target policy $\pi\in\Pi_N$ and the barycenter $\pi_{\mathrm{KL}}$. Specifically, if each $\pi_i\in\Pi_N$ satisfies 
\[ D_{\mathrm{KL}}(\pi_i \mid \pi_{\mathrm{KL}}) \le \eta\,,\]
for some $\eta>0$, then one can establish a tighter upper bound on the weight function, leading to improved sample complexity in our regret analysis.
The following statement gives an upper bound on $\sigma_{\mathrm{KL}}$ in terms of $\eta$. The detailed proof is given in Appendix~\ref{proof:Hellingerprop}.
\begin{proposition}
\label{Hellingerprop}
    Suppose that $D_{\mathrm{KL}}(\pi_i \mid \pi_{\mathrm{KL}})\le \eta$ for all $i\in\{1,\dots,N\}$ and $\pi_{\mathrm{KL}}(a)>0$ for all $a\in\mathcal A$, then it holds that
    \[\sigma_{\mathrm{KL}} \le \min\left(N, 1+\frac{2\eta}{\min_a\pi_{\mathrm{KL}}(a)} + \frac{2\sqrt{2\eta}}{\sqrt{\min_a\pi_{\mathrm KL}(a)}}\right)\,. \]
\end{proposition}
As a result, we can explicitly characterize the sample complexity in terms of $\eta$. For $\delta\in(0,1)$ and $\varepsilon>0$, a sample size 
\[n\ge n(\varepsilon,\delta) = 2 R_\ast^2 \log(N/\delta) \min\left(N, 1+\frac{2\eta}{\min_a\pi_{\mathrm{KL}}(a)} + \frac{2\sqrt{2\eta}}{\sqrt{\min_a\pi_{\mathrm KL}(a)}}\right)^2 \varepsilon^{-2}\]
ensures that $\mathbb P(\mathcal R(\widehat\pi_n)<\varepsilon)\ge 1-\delta$.
A drawback of the derived sample complexity is the fact that it can become large when $\min_a \pi_{\mathrm{KL}}$ is very small. However, this issue can be effectively alleviated by introducing a safe behavior policy that incorporates regularization in the form of 
\[\pi_{\mathrm{safe}}^\lambda(a) := (1-\lambda)\pi_{\mathrm{KL}}(a) + \lambda u(a),\quad a\in\mathcal A, \ \lambda\in(0,1)\,. \]
This approach corresponds to safe or defensive importance sampling \cite{60b7ffd2-e1ec-316c-83e1-d72c47607724}. Choosing $\lambda=\lambda(\eta)=\eta$ ensures that the upper bound on the importance weights remain independent of $\min_a\pi_{\mathrm{KL}}(a)$ as $\eta$ tends to zero. We provide a detailed analysis of the safe behavior policy in 
Appendix~\ref{ap:safe-policy}.

\section{Best policy selection using clustered policy evaluation}
\label{sec:ImprovedEvaluation}

In the previous section, we showed that selecting the best policy based on KL-PE requires certain similarity assumptions (in KL divergence) on the set of target policies. Specifically, we showed that the maximal importance weights can be upper bounded in terms of the KL divergence between each target policy and the KL-barycenter of the set. While ensuring small KL divergences may be challenging in general, a natural approach is to partition the set of target policies into clusters with small KL divergences and then apply KL-PE to each cluster.

In the following, we introduce an improved method that employs multiple behavior policies, each constructed as the KL-barycenters of a subset of $\Pi_N$. Our regret analysis can be extended to this setting, demonstrating that the required sample sizes for achieving a low regret scale with the number of clusters and the maximal weights within each cluster. The specific design of the clusters will be discussed in Section~\ref{sec:experiments}. It is important to highlight that the proposed clustering strategy does not require any additional interactions with the bandit environment.

\subsection{Clustered KL based policy evaluation}\label{sec:clusteredevaluation}
We begin the discussion by proposing an improved structured importance sampling approach using clustered sets of target policies. Decompose the set of target policies $\target$ into $M$ disjoint clusters $K_j := \{\pi_i^{(j)},\ i=1,\dots,N_j\}$ of sizes $N_j$ \footnote{We can easily use cluster algorithms as K-means, see Section \ref{sec:experiments}.} and define
    \begin{equation*}
        \sigma_{\mathrm{c}}^{(j)} := \max_{\ell = 1,\dots, N_j}\max_{a\in\mathcal A} \frac{\pi_\ell^{(j)}(a)}{\frac{1}{N_j}\sum_{i=1}^{N_j}\pi_i^{(j)}(a)}
    \end{equation*}
    for all $j=1,\dots,M$. 
    In our subsequent analysis, we will consider the uniform maximal weight over all clusters defined as 
    \[\sigma_{\mathrm{c}}:=\max_{j=1,\dots,M} \sigma_{\mathrm{c}}^{(j)}\]
    For instance, following Lemma~\ref{lem:sigmaKL} the value $\sigma_{\mathrm{c}}$ can always be bounded by $\sigma_{\mathrm{c}} \le \max_{j=1,\dots,M} N_j$. We will describe a specific clustering algorithm in Section~\ref{sec:clustering} with the overall goal to achieve small values of $\sigma_{\mathrm{c}}^{(j)}$. 
To evaluate the policies within each cluster we construct the clustered importance sampling estimators
\begin{equation}\label{eq:clusterIS}
    \widehat v_n^{(j)}(\pi) := \frac{1}{n_j}\sum_{t=1}^{n_j} \frac{\pi(A_t^{(j)})}{\pi_{\mathrm{KL}}^{(j)}(A_t^{(j)})} R_t^{(j)}\,,\quad \pi\in K_j\,,
\end{equation}
for the KL-barycenters $\pi_{\mathrm{KL}}^{(j)}:= \frac{1}{N_j}\sum_{i=1}^{N_j} \pi_i^{(j)}$ of the corresponding clusters $K_j$, $j\in\{1,\dots,M\}$. Here, $(A_t^{(j)},R_t^{(j)})_{t=1}^{n_j}$ are iid pairs of action and rewards generated by the behavior policy $\pi_{\mathrm{KL}}^{(j)}$ and the overall sample size is $n=\sum_{j=1}^M n_j$. The clustered best policy selection is then defined as
\begin{equation}\label{eq:bestpolicy_with_cluster}
    \widehat \pi_n := \argmax_{j\in\{1,\dots,M\}} \{\max_{\pi\in K_j}\ \widehat v_n^{(j)}(\pi)\}\,.
\end{equation}

\subsection{Clustered regret analysis}
In this section, we derive regret bounds of the proposed clustered best-policy selection approach. For simplicity, we make the following (mainly notational) assumption on the value of the excess risk per cluster. 
\begin{assumption}\label{ass:bestpolicy}
    Let $\pi_\ast = \argmax_{\pi\in\Pi_N}\ v(\pi)$ and $\pi_\ast^{(j)} := \argmax_{\pi\in K_j}\ v(\pi)$ for $j=1,\dots,M$, then (i) $v(\pi_\ast) = v(\pi_\ast^{(1)})$ and (ii) $v(\pi_\ast^{(1)}) - v(\pi_\ast^{(j)})> \Delta$ for some $\Delta>0$ and all $j=2,\dots,M$.
\end{assumption}
In the following, we analyse the excess risk of $\widehat \pi_n$. In order to quantify the upper regret bound, we need to control the probability that the best policy is selected within cluster $K_1$ that contains the best policy. The full proof is given in Appendix~\ref{proof:bestcluster}
\begin{proposition}\label{prop:bestcluster}
    Suppose that Assumption \ref{ass:bestpolicy} is in place, and $n_1=\cdots =n_M$. Furthermore, let $\delta\in(0,1)$ and $\varepsilon\in(0,\Delta]$. If 
    \begin{align*} 
    n = n_1\cdot M\ge n(\varepsilon,\delta) := \frac{2 M R_\ast^2 \sigma_{\mathrm{c}}^2\log\big(\frac{(M-2)(N_1+1) + N + M}{\delta}\big) }{\varepsilon^2}\,,
    \end{align*}
    then $\mathbb P(\widehat \pi_n \notin K_1) \le  \delta\,.$
\end{proposition}

\begin{remark}
   The derived upper bound could be further tightened by leveraging the fact that all policies within a cluster are inherently ``similar." Specifically, we could refine the bound by incorporating the observation that the difference $|\hat{v}(\pi) - \hat{v}(\tilde{\pi})|$ remains small for any pair of policies $\pi$ and $\tilde{\pi}$ within the same cluster. However, this refinement would only lead to an improvement in the logarithmic factor of the bound. Given its limited impact on the overall result, we omit this adjustment for simplicity.
\end{remark}

Having established the high probability guarantee for $\widehat \pi_n\in K_1$, we can derive the overall regret bound of the clustered best-policy selection.  
\begin{theorem}
\label{th:cluster-bound}
    Suppose that Assumption \ref{ass:bestpolicy} is in place, and $n_1=\cdots =n_M$. Furthermore, let $\delta\in(0,1)$,  $\varepsilon\in(0,\Delta]$. If    \begin{align*} 
    n =n_1\cdot M \ge n(\varepsilon,\delta):=  \frac{2MR_\ast^2 \sigma_{\mathrm{c}}^2 \log(\frac{2+N+M+(M-1)(N_1+1)}{\delta})}{\varepsilon^2}\,
    \end{align*}
    then
    \[\mathbb P(\mathcal R(\widehat \pi_n) < \varepsilon) \ge 1-\delta\,.\]
\end{theorem}
We provide the full proof of Theorem~\ref{th:cluster-bound} in Appendix~\ref{proof:cluster-bound}. Note that Theorem~\ref{th:cluster-bound} depends on the problem-dependent minimal gap $\Delta$ assuming that $\epsilon \in (0, \Delta]$. We present a problem-independent version for the expected excess risk in the appendix, see Corollary~\ref{cor:independent}.

\begin{remark}
    For simplicity, we assumed that the number of samples is distributed equally across clusters. However, sample allocation could be optimized by considering the respective $\sigma_{c_j}$, as the effectiveness of IS estimators depends on $\sigma_{c_j}$.
\end{remark}

Let us finish this section with a comparison of sample complexity with and without clustering. First, note that for $M=1$, i.e. no clustering is applied, the sample complexity (ignoring the logarithmic terms) simplifies to $\frac{2R_\ast^2 \sigma_{\mathrm{KL}}^2}{\varepsilon^2}.$ This means that CKL-PE yields a sample complexity improvement if 
\begin{equation}
\label{eq:improvement}
    M\sigma_{\mathrm{c}}^2 <\sigma_{\mathrm{KL}}^2.
\end{equation}    
Thus, the effectiveness of CKL-PE is directly tied to the reduction in $\sigma_{\mathrm{c}}$ compared to $\sigmaKL$. In the next section, we will numerically quantify this effect.

\section{Empirical Evaluation}\label{sec:experiments}
In order to implement Algorithm~\ref{alg:ISPE} we construct a specific clustering approach. Motivated by Proposition~\ref{Hellingerprop}, we seek to achieve small values of $\sigma_{\mathrm{c}}^{(j)}$ by minimizing the KL divergence of all policies within the cluster. 
Therefore, one can adapt standard cluster algorithms to create clusters with small KL divergences and then apply CKL-PE as discussed in the subsequent sub-sections.

\subsection{Clustering with respect to KL}
\label{sec:clustering}
Clustering algorithms are widely used to partition items into groups based on similarity. Typically, these methods operate within a metric space, where distances between elements satisfy standard metric properties. However, in probability spaces, distributions can also be grouped based on divergence measures rather than traditional metrics. A common class of such measures is \emph{f-divergences}, which includes the KL divergence and the squared Hellinger distance. Notably, the squared Hellinger distance has the additional property that its square root defines a proper metric, enabling the use of standard clustering techniques such as \emph{KMeans}. Since both KL and Hellinger distances belong to the family of \emph{f-divergences}, we can leverage clustering results obtained via the Hellinger distance to infer meaningful groupings under KL divergence. A detailed theoretical justification for this approach is provided in~\citet{Cluster}. The description of the Hellinger-based clustering approach can be found in Algorithm~\ref{alg:Hellinger_cluster}. 

\begin{algorithm}[!ht]
\caption{Hellinger-Based Clustering for Behavior Policy Optimization}
\label{alg:Hellinger_cluster}
\begin{algorithmic}[1]
\Require Number of clusters $M$, set of target policies $\Pi_N$
\Ensure Cluster assignments and computed behavior policies $\{\pi_{\mathrm{KL}}^1, \dots, \pi_{\mathrm{KL}}^{M}\}$\smallskip

\State Initialize an empty set $\Pi_{\mathrm{sqrttargets}}$  
\State \textbf{for} $\pi_{\text{target}}^i \in \Pi_N$:
    \State\quad Compute element-wise square root: $\pi_{\text{sqrttarget}}^i \gets \sqrt{\pi_{\text{target}}^i}$
    \State\quad Add $\pi_{\text{sqrttarget}}^i$ to $\Pi_{\mathrm{sqrttargets}}$
\State\textbf{end}
\State Apply KMeans clustering to $\Pi_{\mathrm{sqrttargets}}$:  
\State $\text{clusters} \gets \text{KMeans}(\Pi_{\mathrm{sqrttargets}}, M, \text{Metric} = D^2_{\mathrm{H}})$
\State \textbf{for} $j = 1$ to $M$:
    \State\quad Compute behavior policies $\pi_{\mathrm{KL}}^j$ ensuring it remains a valid probability distribution
\State\textbf{end}
\end{algorithmic}
\end{algorithm}

\subsection{Experiments}

To clearly visualize the effects, we empirically evaluate the proposed clustering approach on an extreme toy example. Our experiment aims to demonstrate that the proposed regret bounds improve when the target policy set exhibits a structured form. To verify this, we test whether \eqref{eq:improvement} holds. For this experiment, we generate sets of softmax target policies of size $N=1000$ in a $100$-armed bandit setting with Gaussian-distributed rewards by sampling weights for each arm and rescaling to softmax policies with a temperature parameter of $1$. The maximal mean reward is set to $3$ for arm $1$, linearly decaying by $0.05$ until arm $100$. A detailed discussion on the construction can be found in Appendix~\ref{sec:detailsexp}. The optimal target policy achieves a value of $2.33$. We compare KL-PE and CKL-PE for different numbers of clusters $M$ applied in Algorithm~\ref{alg:Hellinger_cluster}. Importantly, the total number of samples used remains the same in each approach. When clustering is applied, samples are uniformly distributed across all clusters. First, we record the resulting values of $M\sigma_c^2$ for different numbers of clusters $M$. Recall, for $M = 1$ we recover the case  $\sigma_c = \sigma_{\mathrm{KL}}$. Afterwards, we observe that increasing the cluster size lowers the effect on $\sigma_{\mathrm{c}}$. This underscores the discussion at the end of Section~\ref{sec:clusteredevaluation}. For practitioners it is important to find a suitable cluster size. The results are shown in Figure~\ref{plot:sigmac} (a). They show clustering significantly improves the value of $M\sigma_{\mathrm{c}}^2$ until a cluster size of $M=10$. This suggests that an improved regret is expected by adjusting the number of clusters. Therefore, we evaluate the average regret of the proposed approach across varying sample sizes for KL-PE and CKL-PE with different numbers of clusters. Additionally, we include a Monte Carlo approach as a baseline, corresponding to the case where $M = 1000$. Note that, a minimum sample size of $n=1000$ is required. The results are shown in Figure~\ref{plot:sigmac} (b), where we observe the best performance for $M=10$. Conversely, too large or too small number of clusters result in higher regrets on average. 

\begin{figure}[htb!]    
    \begin{center}
        \includegraphics[width=.49\textwidth,height=4.5cm]{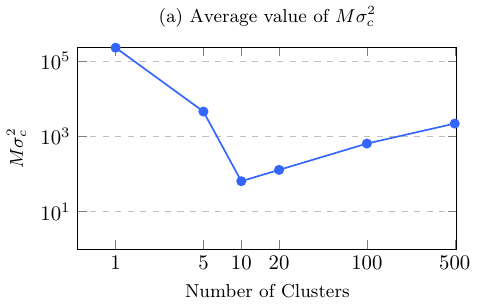}
        \includegraphics[width=.49\textwidth, height=4.47cm]{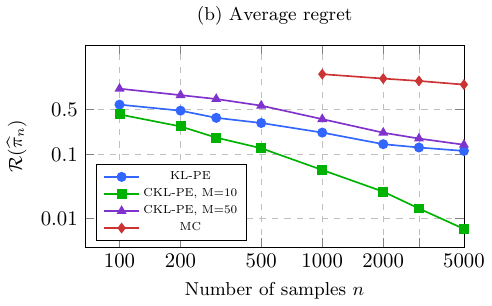}
    \end{center}
    \caption{Comparison of (a) the values of  $M\sigma^2_c$ and (b) the regret for different numbers of clusters $M$ averaged over $1000$ independent runs.}\label{plot:sigmac}
\end{figure}

\section{Conclusion}
In this work, we addressed the question on how to effectively design one or more behavior policies to perform best-policy selection based on importance sampling estimators. Using KL-barycenter behavior policies (with and without clustering of target policies) we provide theoretical and experimental results on the performance. The results focus on the identification of the best policy rather than only obtaining accurate value estimates. By introducing clustering we also provide a simple practical solution to scale up the approach to many target policies, that does not require any additional interaction with the underlying bandit environment. We observe that incorporating clustering significantly reduces regret, highlighting the advantage of cluster-based behavior policies and validating our theoretical findings, providing a positive answer to the questions posed in Question~Q\ref{question}.  

Future research could explore extensions of these theoretical results to contextual bandits and MDPs, such as applications of IS in reinforcement learning \citep{jain2024adaptiveexplorationdataefficientgeneral,papini2024policygradientactiveimportance}. Extensions would require a generalized clustering methodology that can effectively account for variations across different contexts and states. For real-world applications, it would be interesting to see if the theoretical insight can be valuable in the design of behavior policies used to collect data for more efficient offline learning.

Another promising direction for future work is the sequential integration of policy evaluation and selection. Specifically, one could adaptively eliminate underperforming clusters throughout the clustered policy evaluation process, enhancing efficiency of decision-making.

\paragraph{Acknowledgements} 
C. Vernade is funded by the Deutsche Forschungsgemeinschaft (DFG) under both the project 468806714 of the Emmy Noether Programme and under Germany’s Excellence Strategy – EXC number 2064/1 – Project number 390727645. CV also thanks the international Max Planck Research School for Intelligent Systems (IMPRS-IS).

%%%%%%%%%%%%%%% Literaturverzeichnis Typ 1 %%%%%%%%%%%%%%%%%%%%%%%%

\bibliographystyle{abbrvnat} 
\bibliography{mybib}

\appendix

\section{Proofs of Section~\ref{sec:approach}}
The following lemma gives a more detailed formulation of Lemma~\ref{lem:KLbarycenter}.
\begin{lemma}\label{lem:KLbarycenter}
    Let $\target $ be a set of strictly positive sum-normalized set of target policies. We define the arithmetic mean $\overline{\pi}$ as
    \begin{align*}
    \overline{\pi}(a_k) := \frac{1}{N} \sum_{i=1}^N \pi_i(a_k) \quad \text{for every} \quad k \in K.
    \end{align*} Furthermore let $ \KL(\Pi_N | \pi_b) := \frac1N\sum_{i=1}^N \KL(\pi_i, \pi_b)$ be the average right KL divergence. Then it holds true that
    \[
        D_{KL}(\Pi_N | \pi_b) = \left( H(\overline{\pi}) - \frac{1}{N} \sum_{i=1}^N H(\pi_i)\right) + D_{\mathrm{KL}}(\overline{\pi} | \pi_b),
    \]
    where $H(\pi) = - \sum_{a \in \cA} \pi(a) \log \pi(a)$ is the entropy. Furthermore choosing $\pi_b = \overline{\pi}$ minimizes the average right KL divergence.  
\end{lemma}
\begin{proof}
We can compute the average right KL divergence as follows
    \begin{align*}
        &D_{KL}(\Pi_N | p) = \frac{1}{N} \sum_{i=1}^N D_{KL}(\pi_i,\pi_b)\\
        &= \frac{1}{N} \sum_{i=1}^N \sum_{k \in K} \pi_i(a_k) \log \frac{\pi_i(a_k)}{\pi_b(a_k)} \\
        &= \frac{1}{N} \sum_{i=1}^N \sum_{k \in K} \pi_i(a_k) \log \pi_i(a_k)  - \frac{1}{N} \sum_{i=1}^N \sum_{k \in K} \pi_i(a_k) \log \pi_b(a_k) \\
        &= - \frac{1}{N} \sum_{i=1}^N H(\pi_i) - \sum_{k \in K} \log \pi_b(a_k) \frac{1}{N} \sum_{i=1}^N \pi_i(a_k) \\
        &= \left( H(\overline{\pi}) - \frac{1}{N} \sum_{i=1}^N H(\pi_i) \right) - \sum_{k \in K} \log \pi_b(a_k) \overline{\pi}(a_k) - H(\overline{\pi}) \\
        &= \left( H(\overline{\pi}) - \frac{1}{N} \sum_{i=1}^N H(\pi_i) \right) - \sum_{k \in K} \log \pi_b(a_k) \overline{\pi}(a_k) + \sum_{k \in K} \overline{\pi}(a_k) \log \overline{\pi}(a_k) \\
        &= \left( H(\overline{\pi}) - \frac{1}{N} \sum_{i=1}^N H(\pi_i)\right) + D_{KL}(\overline{\pi} | \pi_b)
    \end{align*}
    With this form of the average right KL divergence, we now want to choose p, such that this term is minimized. We can see that we can split the formula into two parts
    \begin{align*}
        D_{KL}(\Pi_N | \pi_b) = \left( H(\overline{\pi}) - \frac{1}{N} \sum_{i=1}^N H(\pi_i)\right) + D_{KL}(\overline{\pi} | \pi_b).
    \end{align*}
    We immediately see that the first part is independent of the choice of $\pi_b$. This implies that we only have to minimize the second term. The KL divergence of two distributions is minimized, if both distributions are the same, because then $\log \frac{\overline{\pi}}{\overline{\pi}} = 0$. Overall this says that $D_{KL}(\Pi_N | \pi_b)$ is minimized for the choice of $\pi_b = \overline{\pi}$.
\end{proof}

\section{Proofs and additional details of Section~\ref{sec:Structured}}\label{app:unclusteredproofs}
In this section, we give the omitted proofs and additional details of Section~\ref{sec:Structured}.
\subsection{Proof of Proposition~\ref{prop:regret_unclustered}}\label{proof:regret_unclustered}
\begin{proof}
    Recall, that by assuming $R_*$-subgaussian rewards, Hoeffding's inequality immediately implies
    \[ \mathbb P(\widehat v_n(\pi)-v(\pi) \ge \varepsilon) \le \exp\left(-\frac{2\varepsilon^2 n}{(R_*\sigma_{\mathrm{KL}})^2}\right)\,.\]
    Hence, given $\delta\in(0,1)$, with probability at least $1-\delta$ we have
    \[v(\pi)\ge \widehat v_n(\pi) - \frac{1}{\sqrt{2n}} R_\ast \sigma_{\mathrm{KL}} \sqrt{\log(\frac{N}{\delta})}\]
    for all $\pi\in\Pi_N$. Similarly, it holds that
    \[\widehat v_n(\pi)\ge v(\pi) - \frac{1}{\sqrt{2n}} R_\ast \sigma_{\mathrm{KL}} \sqrt{\log(\frac{N}{\delta})}\]
    for all $\pi\in\Pi_N$.
    Furthermore, let $\delta\in(0,1)$ and $\widehat\pi_n:= \argmax_{i=1,\dots,N} \widehat v_n(\pi)$, then with probability at least $1-\delta$,
    \begin{align*}
        v(\widehat\pi_n) > \widehat v_n(\pi_n) - \frac{1}{\sqrt{2n}} R_\ast \sigma_{\mathrm{KL}} \sqrt{\log(\frac{N}{\delta})} &\ge \widehat v_n(\pi_\ast) - \frac{1}{\sqrt{2n}} R_\ast \sigma_{\mathrm{KL}} \sqrt{\log(\frac{N}{\delta})}\\
        &> v(\pi_\ast) - \frac{2}{\sqrt{2n}} R_\ast \sigma_{\mathrm{KL}} \sqrt{\log(\frac{N}{\delta})}\,.
    \end{align*}
    It follows that 
    \begin{equation*}
        \mathcal R(\widehat\pi_n) < \frac{\sqrt{2}}{\sqrt{n}} R_\ast \sigma_{\mathrm{KL}}\sqrt{\log(\frac{N}{\delta})},
    \end{equation*}
    with probability at least $1-\delta$.
\end{proof}

\subsection{Proof of Proposition~\ref{prop:lowerbound}}\label{proof:lowerbound}
\begin{proof}
    We define our multi-armed bandit model by 
    $\mathcal A = (a_1,a_2)$, $\pi_1^{(N)} = (1-\frac{1}{N},\frac{1}{N})$ and $\pi_2^{(N)}=\dots=\pi_N^{(N)} = (\frac{1}{N},1-1/N)$. The behavior policy computes as 
    \[\pi_{\mathrm{KL}}(a_1) =  \frac{2(N-1)}{N^2},\ \pi_{\mathrm{KL}}(a_2) = \frac{(N-1)^2+1}{N^2}\,.\]
    Further, we consider deterministic rewards
    \[ R(a_1) := r_1, \ R(a_2) := r_2,\quad r_1 > r_2 \ge 0\,.\]
    By the choice of $r_1,r_2$ as such that the policies satisfy that $v(\pi_1) > v(\pi_2)=\dots = v(\pi_N)$, which guarantees
    \[\mathbb P(\mathcal R(\widehat \pi_n) > \varepsilon) \ge \mathbb P(\widehat v_n(\pi_2) > \widehat v_n( \pi_1)) =\mathbb P(\sum_{t=1}^n \frac{\pi_2(A_t)-\pi_1(A_t)}{\pi_{\mathrm{KL}}(A_t)}R(A_t) > 0) =: \mathbb P(\sum_{t=1}^n Z(A_t) > 0)\]
    for all $0\le\varepsilon \le v(\pi_1)-v(\pi_2)$. Next, we define 
    \[ z_1 = \frac{\pi_2(a_1)-\pi_1(a_1)}{\pi_{\mathrm{KL}}(a_1)}r_1 = \frac{2N-N^2}{2(N-1)} r_1 <0 \] 
    and 
    \[ z_2 = \frac{\pi_2(a_2)-\pi_1(a_2)}{\pi_{\mathrm{KL}}(a_2)}r_2 = \frac{N^2-2N}{(N-1)^2 + 1}r_2 >0\,\] 
    for $N \geq3.$
    Note that it holds that
    \[\mathbb P(Z(A_t) = z_1) = \pi_{\mathrm{KL}}(a_1) = 2 \frac{N-1}{N^2} := x_1 \quad \text{and}\quad \mathbb P(Z(A_t) = z_2) =  \pi_{\mathrm{KL}}(a_2) =1- 2 \frac{N-1}{N^2} :=x_2\,,\]
    such that
    \[ \tilde Z(A_t) := \frac{Z(A_t) - z_1}{z_2-z_1} \sim {\mathrm{Ber}}(x_2)\,.\]
    Now, we define $r_2 = (1-\lambda)r_1$ for $\lambda>0$. Note that, it immediately follows $r_1 > r_2$ and $v(\pi_1) > v(\pi_2).$
    Furthermore, we get that
    \begin{align*}
        z_2 - z_1 &= \frac{1-\frac{2}{N}}{((N-1)^2 + 1)/N^2}(1-\lambda)r_1 + \frac{1-2/N}{2(N-1)/N^2} \\
        &= r_1 (1-\frac{2}{N}) N^2\left(\frac{(1-\lambda)}{((N-1)^2 + 1)} + \frac{1}{2(N-1)}\right).
    \end{align*}
    It then follows
    \[-\frac{z_1}{z_2-z_1} = \frac{1}{2(N-1)} / \left(\frac{1}{2(N-1)} + \frac{1-\lambda}{(N-1)^2 + 1}\right) = \frac{1}{1+2\frac{1-\lambda}{N-1 + \frac{1}{N-1}}}.\]
    
    \begin{equation}
    \begin{split}
    \mathbb P\left(\sum_{t=1}^n \frac{\pi_2(A_t)-\pi_1(A_t)}{\pi_{\mathrm{KL}}(A_t)}R(A_t)> 0\right) &= \mathbb P\left(\sum_{t=1}^n \tilde Z(A_t) > -\frac{nz_1}{z_2-z_1}\right)\\
    &\ge \frac{1}{\sqrt{2n}}\exp(-n \mathrm{D}(\frac{-z_1}{z_2-z_1} \Vert  x_2)), \\
    \end{split}
    \end{equation}
    where we applied Lemma 4.7.2 from~\citet{ash1990information}, also documented for completeness just below this proof.
    
    With the definition of $\mathrm{D}$, we get 
    \[\mathrm{D}(\frac{-z_1}{z_2-z_1} \Vert  x_2)) : = -\frac{z_1}{z_2-z_1} \log\left(-\frac{z_1}{z_2-z_1} \frac{1}{x_2} \right) + (1 - \frac{-z_1}{z_2-z_1}) \log\left(\frac{ 1 -\frac{-z_1}{z_2-z_1}}{1-x_2}\right).\]
    With the choice of $\lambda = 1/N - \frac{1}{N (N-1)} > 0,$ for $N\ge3,$ we get
    \[-\frac{z_1}{z_2-z_1} = \frac{1}{1 + \frac{2}{N}}.\]
    It then follows that
    \begin{equation}\label{eq:aux1} 
    \begin{split}
    \mathbb P\left(\sum_{t=1}^n \frac{\pi_2(A_t)-\pi_1(A_t)}{\pi_{\mathrm{KL}}(A_t)}R(A_t)> 0\right) &= \mathbb P\left(\sum_{t=1}^n \tilde Z(A_t) > -\frac{nz_1}{z_2-z_1}\right)\\
    &\ge \frac{1}{\sqrt{2n}}\exp(-n \mathrm{D}(\frac{-z_1}{z_2-z_1} \Vert  x_2)) \\  
    &\ge \frac{1}{\sqrt{2n}} \exp\left(-n\left( \frac{1}{1 + \frac2N} \log\left(\frac{1}{1-\frac{2}{N^2} + \frac{4}{N^3}}\right)\right) \right) \\
    &= \frac{1}{\sqrt{2n}} \exp\left(-n\left( \frac{1}{1 + \frac2N} \log\left(1 + \frac{\frac{2}{N^2} - \frac{4}{N^3}}{1-\frac{2}{N^2} + \frac{4}{N^3}}\right)\right) \right) \\
    &\ge \frac{1}{\sqrt{2n}} \exp\left(-n\left( \frac{1}{1 + \frac2N} \frac{\frac{2}{N^2} - \frac{4}{N^3}}{1-\frac{2}{N^2} + \frac{4}{N^3}}\right) \right) \\
    &= \frac{1}{\sqrt{2n}} \exp\left(-n\left( \frac{2(N-2)N}{(N+2)^2 (N^2 -2N +2)}\right) \right)  \\
    &\ge \frac{1}{\sqrt{2n}} \exp\left(-\frac{n}2\sigma^2_N \right)
    \end{split}
    \end{equation}
\end{proof}

\begin{lemma}[Lemma 4.7.2 in~\citet{ash1990information}]
For $k>np$ it holds true that
    \label{eq:binbound}
    \[\mathbb{P}(\mathrm{Bin(n,p)} \geq k) \geq \frac{1}{\sqrt{8k(1-k /n)}} \exp(-n \mathrm{D}(k/n \Vert p)), \]
    where $\mathrm{D}(\cdot \Vert \cdot)$ is the Binary entropy function.
\end{lemma}

\subsection{Proof of Proposition~\ref{Hellingerprop}}\label{proof:Hellingerprop}
\begin{proof}
    First, note that the KL divergence can be lower bounded by the squared Hellinger distance in the sense that
    \[ D_H^2(\pi_i,\pi_{\mathrm{KL}}) = \frac12\sum_{a\in\mathcal A} \left(\sqrt{\pi_i(a)}-\sqrt{\pi_{\mathrm{KL}}(a)}\right)^2 \le \frac12 D_{\mathrm{KL}}(\pi_i \mid \pi_{\mathrm{KL}}). \]
    Hence, by assumption it holds that $D_H^2(\pi_i,\pi_{\mathrm{KL}}) \le \eta$ implying that, for all $a\in\mathcal A$,
    \begin{equation} \label{eq:hellbound} \left(\sqrt{\pi_i(a)}-\sqrt{\pi_{\mathrm{KL}}(a)}\right)^2 \le 2\eta\,.
    \end{equation}
    Let $i\in\{1,\dots,N\}$ and $a\in\mathcal A$ be arbitrary. First note that $\frac{\pi_i(a)}{\pi_{\mathrm{KL}}(a)}\le N$ is always satisfied by construction of $\pi_{\mathrm{KL}}$. For $\pi_i(a)\le\pi_{\mathrm{KL}}(a)$ we obviously have $\frac{\pi_i(a)}{\pi_{\mathrm{KL}}(a)}\le 1$. Next, consider the case $\pi_i(a)>\pi_{\mathrm{KL}}(a)$: Using \eqref{eq:hellbound}, we have
    \[ 0\le \sqrt{\pi_i(a)}\le \sqrt{2\eta} + \sqrt{\pi_{\mathrm{KL}}(a)}\,.\]
    This implies that
    \[ \frac{\sqrt{\pi_i(a)}}{\sqrt{\pi_{\mathrm{KL}}(a)}}\le \frac{\sqrt{2\eta} + \sqrt{\pi_{\mathrm{KL}}(a)}}{\sqrt{\pi_{\mathrm{KL}}(a)}}\le 1 +\frac{\sqrt{2\eta}}{\sqrt{\min_{a}\pi_{\mathrm{KL}}(a)}}\,.\]
    The assertion follows by taking the square on both sides.
\end{proof}

\subsection{Upper bound of the weights for the safe policy}
\label{ap:safe-policy}

\begin{proposition}
\label{prop:safe}
    Suppose that $D_{\mathrm{KL}}(\pi_i,\pi_{\mathrm{KL}}) \le \eta$ for all $i\in\{1,\dots,N\}$ and $\pi_{\mathrm{KL}}(a)>0$ for all $a\in\mathcal A$. For any $\lambda\in(0,1)$ it holds that
    \[ \sigma_{\mathrm{safe}} := \max_{a\in\mathcal A}\ \frac{\pi_i(a)}{\pi_{\mathrm{safe}}^\lambda(a)} \le \min\left(\frac{N}{1-\lambda},\frac{1}{1-\lambda} +\frac{2\eta K}{\lambda} + \frac{2\sqrt{2\eta K}}{\sqrt{1-\lambda}\sqrt{\lambda}} \right)\,.\]
    Moreover, suppose that $\eta<1$ and let $\lambda(\eta) := \sqrt{\eta}$, then 
    \[ \sigma_{\mathrm{safe}} \le \min\left(\frac{N}{1-\sqrt{\eta}},\frac{1}{1-\sqrt{\eta}} + 2\sqrt{\eta} K + \frac{2\sqrt{2K\sqrt{\eta}}}{\sqrt{1-\sqrt{\eta}}}\right)\,.\]
\end{proposition}
\begin{proof}
    Recall, that the KL divergence is lower bounded by the squared Hellinger distance. Let $i\in\{1,\dots,N\}$ and $a\in\mathcal A$ be arbitrary. First, in the case $\pi_i(a)\le \pi_{\mathrm{KL}}(a)$, we have
    \[ \frac{\pi_i(a)}{\pi_{\mathrm{safe}}^\lambda(a)} \le \frac{\pi_{\mathrm{KL}}(a)}{(1-\lambda)\pi_{\mathrm{KL}}(a)+\lambda /K}\le \frac{1}{1-\lambda}\,.\]
    In the case 
    $\pi_i(a)>\pi_{\mathrm{KL}}(a)$, it holds that
    \[ 0\le \sqrt{\pi_i(a)}\le \sqrt{2\eta} + \sqrt{\pi_{\mathrm{KL}}(a)}\,,\]
    and therefore, we have
    \begin{align*}
        \frac{\pi_i(a)}{\pi_{\mathrm{safe}}^\lambda(a)} &\le \frac{2\eta+2\sqrt{2\eta}\sqrt{\pi_{\mathrm{KL}}(a)}+\pi_{\mathrm{KL}}(a)}{(1-\lambda)\pi_{\mathrm{KL}}(a) + \lambda/K}\\ 
        &\le \frac{2\eta K}{\lambda} + \frac{2\sqrt{2\eta}\sqrt{\pi_{\mathrm{KL}}(a)}}{\sqrt{(1-\lambda)\pi_{\mathrm{KL}}(a) + \lambda/K} \sqrt{\lambda/K}} + \frac{\pi_{\mathrm{KL}}(a)}{(1-\lambda)\pi_{\mathrm{KL}}(a) + \lambda/K}\\
        &\le \frac{2\eta K}{\lambda} + 2\sqrt{2\eta}\frac{\sqrt{(1-\lambda)\pi_{\mathrm{KL}}(a)}}{\sqrt{(1-\lambda)\pi_{\mathrm{KL}}(a) + \lambda/K}} \frac{\sqrt{K}}{\sqrt{1-\lambda}\sqrt{\lambda}} + \frac{1}{1-\lambda}\\
        &\le \frac{2\eta K}{\lambda}+2\sqrt{2\eta}\frac{\sqrt{K}}{\sqrt{1-\lambda}\sqrt{\lambda}}+\frac{1}{1-\lambda}
    \end{align*}
    which verifies the first claim. The second claim is a direct consequence.
\end{proof}
\subsection{Scaling-Up Structured Policy Evaluation}\label{app:scaling_up}
This section was omitted in the main paper due to space constraints. It can serve as an additional motivation, why we need clustering to effectively evaluate many target policies. Therefore, we analyze how an increasing number of target policies affects the KL-barycenter. Specifically, we prove that as the number of target policies approaches infinity, the KL-barycenter converges in probability to a uniform policy. To explore this, we consider the scenario where the set of target policies is generated randomly. Practically, this corresponds to having no prior knowledge about the policies to be evaluated.

For each target policy, we sample $ K $ independent identically distributed weights (where $K$ is the number of arms) according to some probability law. Since we focus on target policies that assign positive probabilities to every arm, we use the widely adopted softmax policies with a weighting parameter $\tau$. Formally, for each target policy $i \in \{1, \ldots, N\} $, we sample $K$ independent identically distributed weights $X^{(l)}_i$, with  $l \in \{1, \ldots, K\}$, and compute the probability for each arm as: 
\[
\pi_i(a) = \frac{\exp(\tau X^{(a)}_i)}{\sum_{l=1}^K \exp(\tau X^{(l)}_i)}.
\]
Using this construction, we prove the following result as the number of target policies grows to infinity.

\begin{lemma}
    \label{lem:uniform}
    Let the weights for each arm be drawn independent identically distributed according to some probability distribution. Let the weighting parameters $\tau \in \Theta = [0,L]^K$ for some $L \in \mathbb{N}$. Then, the probability weight of the behavior policy associated with each arm $l$ for $l \in \{1, \ldots, K\}$ converges in probability to $\frac{1}{K}$ as the number of target policies $N$ goes to infinity. Formally:
    \begin{align*}
    \pi_{\mathrm{KL}}(a) :&= \frac{1}{N}\sum_{i=1}^{N} \pi_i(a) = \frac{1}{N} \sum_{i=1}^{N} \frac{\exp{(\tau X_i^{(l)})}}{\sum_{l=1}^{K} \exp{(\tau X_i^{(l)})}} \\
    &= \frac{1}{N} \sum_{i=1}^{N} f(X^{(l)}, \tau) \overset{P}{\rightarrow} \mathbb{E}\left[ f(X^{(l)}, \tau)\right] = \frac{1}{K}.
    \end{align*}
\end{lemma}

However, a uniform sampling policy is inefficient and undesirable as a sampling or behavior policy. This motivates the need for new techniques to scale up the evaluation of the target policy set effectively. One particularly promising approach is the use of clustering. 

To prove this lemma we first introduce some intermediate results. 
In particular, we need a uniform law of large numbers as well a result on the expectation of normalized random variables. First, we give the uniform law of large numbers. A proof of this theorem can be found in \cite{uniform}.
\begin{theorem}
    \label{thm:uniform}
    Let a sequence of independent identically distributed random variables $X_1, X_2 \ldots, X_n$ according to some probability law $\mathbb{P}_X$ be given. Furthermore, let $\Theta$ be a compact space and $f$ an a.s. integrable function with $f(X, \tau)$ continuous at each $\tau \in \Theta$ with probability 1. 
    If we additionally assume that there exists a dominating function $d(x)$ with $\mathbb{E}\left[d(X)\right] < \infty$ and
        \begin{align*}
            \|f(x,\tau)\| \leq d(x) \quad \text{for all}  \quad \tau \in \Theta.
        \end{align*}
    Then  it holds true $\mathbb{E}\left[ f(x, \tau)\right]$ is continuous in $\tau$, and
    \begin{align*}
        \sup_{\tau \in \Theta} \left\| \frac{1}{n} \sum_{i=1}^{n} f(X_i, \tau) - \mathbb{E}\left[f(X, \tau)\right]\right\| \overset{P}{\to} 0
    \end{align*}
\end{theorem}

\begin{lemma}
\label{lem:1/k}
    Let $X_1, \ldots, X_K$ be independent, identically distributed non-negative random variables. Then the for every $k \in \{1, \ldots, K\}$ it holds that
    $\mathbb{E}\left[ \frac{X_k}{\sum_{i=1}^{K} X_i}\right] = \frac{1}{K}$.
\end{lemma}
\begin{proof}
    In the first step we use that all the random variables are identically distributed, therefore we get for an arbitrary value $m \in \mathbb{R}$
    \[m = \mathbb{E}\left[ \frac{X_1}{\sum_{i=1}^{K} X_i} \right] = \ldots =  \mathbb{E}\left[ \frac{X_K}{\sum_{i=1}^{K} X_i} \right].\]
    Multiplying $m$ by the number of non negative random variables $K$ we get 
    \[Km = \mathbb{E}\left[ \frac{X_1}{\sum_{i=1}^{K} X_i} \right] + \ldots +  \mathbb{E}\left[ \frac{X_K}{\sum_{i=1}^{K} X_i} \right]\]
    With the linearity of expectations we get 
    \[Km = \mathbb{E}\left[ \frac{X_1}{\sum_{i=1}^{K} X_i} \right] + \ldots +  \mathbb{E}\left[ \frac{X_K}{\sum_{i=1}^{K} X_i} \right] = \mathbb{E}\left[\frac{\sum_{i=1}^{K} X_i}{\sum_{i=1}^{K} X_i}\right] = 1.\]
    Rewriting this equation we get that the value of the expectation is $m = \frac{1}{K}$.
\end{proof}

With these results, we can prove Lemma~\ref{lem:uniform}.

\begin{proof}[Lemma~\ref{lem:uniform}]
Since $\Theta$ is a compact space and $f$ is continuous for all $\theta \in \Theta$ with probability 1. Furthermore, we can bound $f(X, \tau)$ by 1 for all $\tau \in \Theta$. Therefore, we can apply Theorem~\ref{thm:uniform}. Therefore we even have a uniform convergence in probability which implies the point wise convergence in probability above. Additionally, we know by Lemma~\ref{lem:1/k} that the expectation is given by $\frac{1}{K}$.
\end{proof}

\section{Proofs and additional details of Section~\ref{sec:ImprovedEvaluation}}
In this section, we provide the omitted proofs and additional discussions of Section~\ref{sec:ImprovedEvaluation}.
\subsection{Proof of Proposition~\ref{prop:bestcluster}}\label{proof:bestcluster}
\begin{proof}
    First, observe that
    \begin{align*} 
    \mathbb P(\widehat \pi_n \notin K_1)&= \mathbb P(\underset{j=2,\dots,M}{\cup}\{ \max_{\pi\in K_j}\ \widehat v_n^{(j)}(\pi) - \max_{\pi\in K_1}\ \widehat v_n^{(1)}(\pi) \ge \Delta\})\\ &\le \sum_{j=2}^M \mathbb P(\max_{\pi\in K_j}\ \widehat v_n^{(j)}(\pi) - \max_{\pi\in K_1}\ \widehat v_n^{(1)}(\pi) \ge \min(\Delta,\varepsilon))\\
    &\le \sum_{j=2}^M \mathbb P(\max_{\pi\in K_j}\ \widehat v_n^{(j)}(\pi) - \max_{\pi\in K_1}\ \widehat v_n^{(1)}(\pi) \ge \varepsilon)
    \end{align*}
    and we proceed by considering each probability in the last line separately. Let $j\in\{2,\dots,M\}$ and decompose as follows
    \begin{align*}
        \max_{\pi\in K_j}\ \widehat v_n^{(j)}(\pi) - \max_{\pi\in K_1}\ \widehat v_n^{(1)}(\pi) & = \max_{\pi\in K_j}\ \widehat v_n^{(j)}(\pi) -v(\pi_\ast^{(j)}) + v(\pi_\ast^{(j)}) - v(\pi_\ast^{(1)})\\ &\quad +v(\pi_\ast^{(1)})-\max_{\pi\in K_1}\ \widehat v_n^{(1)}(\pi)\\
        &\le \max_{\pi\in K_j}\ \widehat v_n^{(j)}(\pi) -v(\pi_\ast^{(j)}) + v(\pi_\ast^{(1)})-\max_{\pi\in K_1}\ \widehat v_n^{(1)}(\pi)
    \end{align*}
    almost surely due to the assumption $v(\pi_\ast^{(j)}) - v(\pi_\ast^{(1)})<0$. Hence, we have
    \begin{align*}
        \mathbb P(\max_{\pi\in K_j}\ \widehat v_n^{(j)}(\pi) - \max_{\pi\in K_1}\ \widehat v_n^{(1)}(\pi) \ge \varepsilon)&\le \mathbb P(|\max_{\pi\in K_j}\ \widehat v_n^{(j)}(\pi)-v(\pi_\ast^{(j)})|\ge \varepsilon/2 )\\ &\quad +\mathbb P(|v(\pi_\ast^{(1)})- \max_{\pi\in K_1}\ \widehat v_n^{(1)}(\pi)|\ge \varepsilon/2)\,.
    \end{align*} 
   For all $j\in\{1,\dots,M\}$ it holds that
        \begin{align*} 
    \mathbb P(|\max_{\pi\in K_j}\ \widehat v_n^{(j)}(\pi)\!-\!v(\pi_\ast^{(j)})|\!\ge\! \tfrac{\varepsilon}{2})
    &\le \mathbb P(\max_{\pi\in K_j}\ \widehat v_n^{(j)}(\pi)\!-\!v(\pi_\ast^{(j)})\ge \tfrac{\varepsilon}{2})
    \!+\! \mathbb P(v(\pi_\ast^{(j)})\!-\!\max_{\pi\in K_j} \widehat v_n^{(j)}(\pi)\!\ge\! \tfrac{\varepsilon}{2})\\
     &\le \mathbb P(\underset{\pi\in K_j}{\cup} \{\widehat v_n^{(j)}(\pi)-v(\pi_\ast^{(j)})\ge \tfrac{\varepsilon}{2}\})
     \!+\! \mathbb P(v(\pi_\ast^{(j)})-\widehat v_n^{(j)}(\pi_\ast^{(j)})\ge \tfrac{\varepsilon}{2})\\
    &\le (N_j+1) \exp(-\frac{ \varepsilon^2 n_j}{2 R_\ast^2 \sigma_{\mathrm{c}}^2})\\
    &= (N_j+1) \exp(-\frac{ \varepsilon^2 n}{2M R_\ast^2 \sigma_{\mathrm{c}}^2})\,,
    \end{align*}
    where we have used Hoeffding's inequality for $R_\ast$-subgaussian rewards. 
    In total we obtain
    \begin{align*} 
    \mathbb P(\widehat \pi_n \notin K_1) &\le  (M-1) (N_1+1) \exp(-\frac{ \varepsilon^2 n}{2 M R_\ast^2 \sigma_{\mathrm{c}}^2}) + \exp(-\frac{ \varepsilon^2 n}{2M R_\ast^2 \sigma_{\mathrm{c}}^2})\sum_{j=2}^M (N_j+1) \\
    &= \exp(-\frac{ \varepsilon^2 n}{2M R_\ast^2 \sigma_{\mathrm{c}}^2}) (N + M + (M-2)(N_1+1)
    \end{align*}
    which is bounded by $\delta$ by the choice $n\ge n(\varepsilon,\delta)$.
\end{proof}

\subsection{Proof of Theorem~\ref{th:cluster-bound}}\label{proof:cluster-bound}
\begin{proof}
    Let $\pi_\ast = \argmax_{\pi\in\Pi_N}\ v(\pi)$ and $\widehat\pi_n$ be defined in \eqref{eq:bestpolicy_with_cluster}. Using the law of total probability we obtain
    \begin{align*}
        \mathbb P(v(\pi_\ast) - v(\widehat \pi_n ) \ge \varepsilon) &= \mathbb P(\widehat \pi_n\notin K_1)\mathbb P(v(\pi_\ast) - v(\widehat \pi_n ) \ge \varepsilon\mid \widehat \pi_n\notin K_1) \\
        &\quad + \mathbb P(\widehat \pi_n\in K_1)\mathbb P(v(\pi_\ast) - v(\widehat \pi_n ) \ge \varepsilon\mid \widehat \pi_n\in K_1)\\
        &\le \mathbb P(\widehat \pi_n\notin K_1) + \mathbb P(v(\pi_\ast) - v(\widehat \pi_n ) \ge \varepsilon\mid \widehat \pi_n\in K_1)\\
        &\le \mathbb P(\widehat \pi_n\notin K_1) + \mathbb P(|v(\pi_\ast) - \widehat v_n^{(1)}(\pi_\ast )| \ge \varepsilon/2 \mid \widehat \pi_n\in K_1)\\
        &\quad+ \mathbb P(|\widehat v_n^{(1)}(\pi_\ast ) - v(\widehat \pi_n )| \ge \varepsilon/2 \mid \widehat \pi_n\in K_1).
    \end{align*}
    Firstly, by the proof of Proposition~\ref{prop:bestcluster} we have 
    \begin{equation*}
        \mathbb P(\widehat \pi_n\notin K_1)\le
        (N + M + (M-2)(N_1+1))\exp(-\frac{ \varepsilon^2 n}{2M R_\ast^2 \sigma_{\mathrm{c}}^2}) \,.
    \end{equation*} Secondly, due to Assumption~\ref{ass:bestpolicy}, $\pi_\ast\in K_1$, we have
    \begin{equation*} 
    \mathbb P(|v(\pi_\ast)-\widehat v_n^{(1)}(\pi_\ast)|\ge \varepsilon/2 \mid \pi_n\in K_1)= \mathbb P(|v(\pi_\ast)-\widehat v_n^{(1)}(\pi_\ast)|\ge \varepsilon/2)\le 2\exp(-\frac{\varepsilon^2 n}{2 MR_\ast^2\sigma_{\mathrm{c}}^2}),
    \end{equation*}
    where we have again used Hoeffding's inequality and the choice $n = n(\varepsilon,\delta)$. Finally, we consider the last term
    \begin{align*}
        \mathbb P(|\widehat v_n^{(1)}(\pi_\ast ) - v(\widehat \pi_n )| \ge \varepsilon/2 \mid \widehat \pi_n\in K_1)&\le \mathbb P(\widehat v_n^{(1)}(\pi_\ast ) - v(\widehat \pi_n ) \ge \varepsilon/2 \mid \widehat \pi_n\in K_1)\\
        &\quad + \mathbb P(v(\widehat \pi_n )-\widehat v_n^{(1)}(\pi_\ast ) \ge \varepsilon/2 \mid \widehat \pi_n\in K_1)\,.
    \end{align*}
    Using the fact $\widehat v_n^{(1)}(\pi_\ast) \le \widehat v_n^{(1)}(\widehat \pi_n)$ conditioned on $\widehat \pi_n \in K_1$, it holds that
    \begin{align*}
        \mathbb P(\widehat v_n^{(1)}(\pi_\ast ) - v(\widehat \pi_n ) \ge \varepsilon/2 \mid \widehat \pi_n\in K_1)
        &\le \mathbb P(\widehat v_n^{(1)}(\widehat\pi_n ) - v(\widehat \pi_n ) \ge \varepsilon/2 \mid \widehat \pi_n\in K_1)\\
        &\le \mathbb P(\underset{\pi\in K_1}{\cup}\{\widehat v_n^{(1)}(\pi ) - v(\pi ) \ge \varepsilon/2\} \mid \widehat \pi_n\in K_1)\\
        & = \mathbb P(\underset{\pi\in K_1}{\cup}\{\widehat v_n^{(1)}(\pi ) - v(\pi ) \ge \varepsilon/2\})\\
        &\le \sum_{\pi\in K_1 } \mathbb P(\widehat v_n^{(1)}(\pi ) - v(\pi ) \ge \varepsilon/2)\\
        &\le N_1 \exp(-\frac{\varepsilon^2 n}{2MR_\ast^2\sigma_{\mathrm{c}}^2})\,,
    \end{align*}
    where we have used Hoeffding's inequality. Similarly, since $v(\widehat \pi_n)\le v(\pi_\ast)$ almost surely, by Hoeffding's inequality we obtain 
    \begin{align*}
        \mathbb P(v(\widehat \pi_n )-\widehat v_n^{(1)}(\pi_\ast ) \ge \varepsilon/2 \mid \widehat \pi_n\in K_1)
        &\le \mathbb P(v(\pi_\ast)-\widehat v_n^{(1)}(\pi_\ast ) \ge \varepsilon/2 \mid \widehat \pi_n\in K_1)\\
        &= \mathbb P(v(\pi_\ast)-\widehat v_n^{(1)}(\pi_\ast ) \ge \varepsilon/2)\\
        &\le \exp(-\frac{\varepsilon^2 n}{2MR_\ast^2\sigma_{\mathrm{c}}^2}).
    \end{align*}
    Overall, we proved that 
    \begin{align*}
        \mathbb P(v(\pi_\ast) - v(\widehat \pi_n ) \ge \varepsilon) &\le (N_1+3+N+M+(M-2)(N_1+1)) \exp(-\frac{\varepsilon^2 n}{2MR_\ast^2\sigma_{\mathrm{c}}^2})\\
        & = (2+N+M+(M-1)(N_1+1))\exp(-\frac{\varepsilon^2 n}{2MR_\ast^2\sigma_{\mathrm{c}}^2}).
    \end{align*}
    Choosing $n=n(\varepsilon,\delta) = \frac{2MR_\ast^2 \sigma_{\mathrm{c}}^2 \log(\frac{2+N+M+(M-1)(N_1+1)}{\delta})}{\varepsilon^2}$ the righthand side simplifies to $\delta$ and the claim follows.
\end{proof}

\subsection{Problem-independent expected regret bound}\label{sec:PI_bound}
\begin{corollary}
\label{cor:independent}
    Suppose that Assumption~\ref{ass:bestpolicy} is in place, and $n_1=\cdots =n_M$. For any $M\in\{1,\dots,N\}$ the regret is bounded by 
    \[
    \mathbb{E}\left[\mathcal{R}(\hat{\pi}_n))\right]\le \frac{\Delta_{\max}}{\sqrt{n}}\left(1+ \frac{MN_1}{N}+\frac{2M}{N}\right) + \sqrt{2}M^{3/2}R_\ast \sigma_c \sqrt{\frac{\log(N\sqrt{n})}{n}},
    \]
    where $n=M\cdot n_1$.
\end{corollary}
\begin{proof}
    Let $\Delta_j := v(\pi_\ast)-v(\pi_\ast^{(j)})$, $j=\{1,\dots,M\}$. Without loss of generality assume that $\Delta_1 = 0$ and $\Delta_j>0$ for all $j=\{2,\dots,M\}$
    with $\Delta_j\leq\max_j \Delta_j=:\Delta_{\max}$. For arbitrary $\eta>0$, we have
\begin{align*}
    \mathbb{E}\left[\mathcal{R}(\hat{\pi}_n))\right] &= \mathbb E[v(\pi_*) - v(\widehat \pi_n)] \\
    & =\sum_{j:\Delta_j>\eta} \mathbb E[(v(\pi_*) - v(\widehat \pi_n))\mathds{1}_{\widehat\pi_n\in K_j}] + \sum_{j: \Delta_j<\eta} \mathbb E[(v(\pi_*) - v(\widehat \pi_n))\mathds{1}_{\widehat\pi_n\in K_j}]\\
    &\leq \sum_{j: \Delta_j>\eta }  \Delta_{\max}  \mathbb{P}(\hat{\pi}_n \in K_j )+ \sum_{j: \Delta_j<\eta} \eta \mathbb{P}(\hat{\pi}_n \in K_j)\\
    &\leq \sum_{j: \Delta_j>\eta }  \Delta_{\max} \mathbb{P}(\max_{\pi \in K_j}\hat{v}_n(\pi)> \max_{\pi \in K_1} \hat{v}_n(\pi) )+ \eta M,
\end{align*}
where we have used $\sum_{j=1}^M \mathds{1}_{\widehat\pi_n\in K_j} = 1$ almost surely. 
By the proof of Proposition~\ref{prop:bestcluster}, for $j\in\{2,\dots,M\}$ with $\Delta_j>\eta$ we have 
\begin{align*}
    \mathbb{P}(\max_{\pi \in K_j}\hat{v}_n(\pi)> \max_{\pi \in K_1} \hat{v}_n(\pi) ) \leq (N_j+N_1+2) \exp\left( - \frac{\eta^2 n}{2MR^2_* \sigma^2_c}\right). 
\end{align*}
Thus, for the choice of 
\[
\eta = \sqrt{\frac{\log(N\sqrt{n})2MR_*^2 \sigma_c^2}{n}} 
\]

we achieve an overall expected regret bound
\begin{align*}
    \mathbb{E}\left[\mathcal{R}(\hat{\pi}_n))\right] 
   & \leq \sum_{j: \Delta_j>\eta }\Delta_{\max}  (N_j+N_1+2) \exp\left( - \frac{\eta^2 n}{2MR^2_* \sigma^2_c}\right) + \eta M\\
   & \leq \Delta_{\max} (N+MN_1+2M) \exp\left( - \frac{\eta^2 n}{2MR^2_* \sigma^2_c}\right) + \eta M\\
    &\le \frac{\Delta_{\max}}{\sqrt{n}}\left(1+ \frac{MN_1}{N}+\frac{2M}{N}\right) + \sqrt{2}M^{3/2}R_\ast \sigma_c \sqrt{\frac{\log(N\sqrt{n})}{n}}\,.
\end{align*}

\end{proof}

\section{Details on Experimental Setup}
\label{sec:detailsexp}

In this section, we provide details on the experimental setup. The experiments are conducted by following the procedure below:

\begin{enumerate}
    \item We construct a multi-armed bandit environment by defining a reward distribution and the number of arms. The reward distribution for each arm follows a Gaussian distribution. Specifically, we set the number of arms to $100$, with the highest mean reward of $3$ assigned to arm $1$, decreasing linearly by $0.05$ per arm until arm $100$. The variance is sampled uniformly and independent from $(1,3)$. 
    
    \item We set the number of target policies to $N=1000$. For each policy and each arm, we sample a weight uniformly from $(1,2)$. To introduce structured dependencies among target policies, we form groups of policies that prioritize specific arms by adding an additional random weight sampled uniformly from $(1,10)$ on those arms. In particular, we create $6$ groups of the following sizes $[25,50,25,825,50, 25]$, with preferred arms $[[2], [3,5], [22,24,34], [23,99], [99],[53]]$. To ensure that every policy assigns positive probability to all arms, we transform the sampled weights into softmax policies using a temperature parameter of $1$. 
    
    \item Given the set of target policies, we apply Algorithm~\ref{alg:Hellinger_cluster} to obtain the KL-barycenters corresponding to different clusters. 
    
    \item Finally, we compute the importance sampling estimates for all target policies using a fixed sample size. In the case of clustering, the samples are distributed uniformly across the clusters. 
\end{enumerate}

\end{document}